\documentclass{article}

\usepackage{arxiv}

\usepackage[utf8]{inputenc} % allow utf-8 input
\usepackage[T1]{fontenc}    % use 8-bit T1 fonts
\usepackage{hyperref}       % hyperlinks
\usepackage{url}            % simple URL typesetting
\usepackage{booktabs}       % professional-quality tables
\usepackage{amsfonts}       % blackboard math symbols
\usepackage{nicefrac}       % compact symbols for 1/2, etc.
\usepackage{microtype}      % microtypography
\usepackage{lipsum}		% Can be removed after putting your text content
\usepackage{graphicx}
\usepackage{natbib}
\usepackage{doi}

\usepackage{graphicx}
\usepackage{amsmath}
\usepackage{amssymb}
\usepackage{hyperref}
\usepackage{booktabs}

\usepackage[ruled,vlined]{algorithm2e} % if you use the two before comment this
\usepackage{amsthm}
\theoremstyle{plain}% default
\newtheorem{thm}{Theorem}[section]

\newtheorem{prop}[thm]{Proposition}

\theoremstyle{definition}
\newtheorem{defn}{Definition}[section]

\theoremstyle{remark}

\usepackage{times}
\usepackage{enumitem} % controlling list spacing
\usepackage{lipsum}
\title{CH-MARL: Constrained Hierarchical Multiagent Reinforcement Learning for Sustainable Maritime Logistics}

\author{ \href{https://orcid.org/0000-0002-2111-3456}{\includegraphics[scale=0.06]{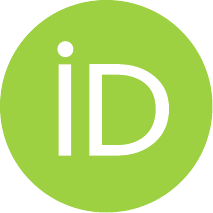}\hspace{1mm}Saad~Alqithami}%\thanks{} 
\\
	\texttt{alqithami@gmail.com} \\
}

\renewcommand{\headeright}{---DRAFT: Under-review}
\renewcommand{\undertitle}{Technical Report}
\renewcommand{\shorttitle}{\textit{Chmarl}}

\hypersetup{
pdftitle={CHMARL},
pdfsubject={CS},
pdfauthor={Saad~Alqithami},
pdfkeywords={First keyword, Second keyword, More},
}

\begin{document}
\maketitle

\begin{abstract}
Constrained multiagent reinforcement learning (MARL) offers a powerful paradigm for coordinating decisions among autonomous agents, yet few approaches address the simultaneous need for global emissions control, partial observability, and fairness in dynamic industrial settings. In this paper, we propose a novel Constrained Hierarchical MARL (CH-MARL) framework for optimizing energy efficiency and reducing greenhouse gas emissions within maritime logistics. Our method formulates the environment as a partially observable, non-stationary system in which vessel agents, port agents, and regulatory agents cooperate under global environmental caps. Specifically, we extend conventional policy-gradient techniques by introducing a real-time constraint-enforcement layer that dynamically adjusts the agents’ feasible action space and shared reward signals to ensure overall compliance with environmental targets. To handle the inherent complexity of maritime operations, we adopt a hierarchical approach: high-level agents learn strategic decisions such as route planning and emission budgeting, while lower-level agents fine-tune local actions (e.g., speed control, berth scheduling). We further embed a fairness-aware objective to balance resource allocation among vessels of different sizes and capacities, preventing disproportionate costs for smaller stakeholders. Experimental evaluations on a digital-twin testbed of multiple shipping lanes and variable port conditions demonstrate that CH-MARL outperforms baseline methods by significantly reducing total emissions and fuel consumption without compromising operational throughput. Moreover, the real-time constraint layer consistently maintains global emissions below specified limits, and fairness metrics indicate minimal disparities among agents in fuel cost and delay. Our findings highlight the scalability and adaptability of CH-MARL for sustainability-driven maritime logistics and pave the way for broader applications in other constrained, multi-objective industrial domains.
\end{abstract}

% keywords can be removed

\section{Introduction} \label{introduction}

%The ongoing expansion of global trade has led to unprecedented increases in both the scale and complexity of maritime logistics. As one of the most cost-effective methods for transporting goods internationally, maritime shipping has become vital to sustaining economic growth worldwide. However, this rapid progress entails substantial environmental and logistical hurdles. Due to extensive reliance on fossil fuels, the shipping industry accounts for approximately 2.89\% of global greenhouse gas (GHG) emissions \cite{smith2014third,imo2020}. 
The advent of globalized trade has led to unprecedented growth in the volume and complexity of maritime logistics. As one of the most cost-effective modes of transportation, maritime shipping has become indispensable for connecting economies and supporting international trade. However, this growth comes with substantial environmental and operational challenges. The sector’s heavy reliance on fossil fuels contributes significantly to global greenhouse gas (GHG) emissions, accounting for nearly 2.89\% of global emissions \cite{smith2014third,imo2020}. Moreover, the International Maritime Organization (IMO) has outlined a strategy to reduce GHG emissions from international shipping by at least 50\% by 2050 compared to 2008 levels, aiming for eventual decarbonization \cite{imo2018strategy}. These ambitious targets underscore the pressing need for transformative solutions to meet regulatory requirements and societal expectations.

%In response, the International Maritime Organization (IMO) has established strict targets to reduce maritime GHG emissions by at least 50\% by 2050 compared to 2008 levels, with an eventual aim of reaching complete decarbonization \cite{imo2018strategy}. Meeting these goals poses a formidable challenge that requires transformative strategies to ensure long-term regulatory compliance and to address rising environmental concerns.

%Alongside these environmental considerations, maritime logistics involves diverse stakeholders (i.e., shipping companies, port authorities, and policymakers) each operating under distinct objectives and resource limitations. Efficiently coordinating such a multifaceted ecosystem requires balancing energy efficiency, emissions reduction, and operational throughput within a single framework \cite{fan2019multi}. Recent advances in AI and optimization underscore the potential of computational methods to tackle these challenges.
Environmental pressures are further compounded by the intricate logistics of coordinating diverse stakeholders, including shipping companies, port authorities, and policymakers, each with unique objectives and constraints. Maritime shipping often operates in complex, multi-actor environments, where optimizing energy efficiency, reducing emissions, and maintaining operational throughput are competing priorities. %\cite{fan2019multi,  }. 
Addressing these multifaceted issues requires innovative strategies that leverage advancements in artificial intelligence and optimization techniques.

%\emph{Multiagent Reinforcement Learning} (MARL) has shown promise in coordinating decision-making among autonomous agents across large and dynamic systems \cite{hernandez2019survey}. Nevertheless, the integration of sustainability and fairness objectives into MARL frameworks adds layers of complexity with respect to meeting global emission caps and providing equitable outcomes for all stakeholders \cite{ ,shi2020multi}.
The field of Multiagent Reinforcement Learning (MARL) offers a promising approach for tackling the dynamic and distributed nature of maritime logistics. By enabling multiple autonomous agents to interact and learn in complex environments, MARL provides a framework for optimizing decision-making processes across scales \cite{hernandez2019survey}. However, integrating sustainability and fairness considerations into such frameworks introduces additional layers of complexity, particularly when navigating global constraints like emission caps and ensuring equitable outcomes among stakeholders.% \cite{ , shi2020multi}.

\subsection{Motivation}
Maritime logistics underpins around 80\% of global trade, highlighting its role as the cornerstone of international commerce \cite{imo2020}. Rising trade volumes coincide with growing urgency to reduce the sector’s environmental impact, most notably by curbing GHG emissions \cite{imo2020}. Regulatory bodies, alongside greater societal awareness of climate change, have amplified demands for cleaner shipping and sustainable maritime operations.% \cite{ }. 

Hence, stakeholders across the maritime sector, from international shipping conglomerates to smaller shipping lines, are compelled to adopt more sustainable practices. Optimizing energy efficiency, reducing emissions, and maintaining operational throughput simultaneously is challenging in a competitive, dynamic environment. %\cite{lin2021hierarchical}
These complexities motivate the development of novel algorithms and frameworks that harness MARL to address environmental imperatives without compromising economic viability.

\subsection{Problem Statement}
Achieving sustainability in maritime logistics involves optimizing multiple objectives concurrently. Stakeholders must ensure adherence to global emission caps while preserving efficiency and upholding fairness in resource allocation. Such objectives become more complicated in real-world maritime environments characterized by partial observability, non-stationarity, and the interactions of numerous autonomous agents \cite{hernandez2019survey,foerster2016learning}.

Existing approaches tend to address particular aspects of these intricacies. They may focus on localized optimization or partial constraint enforcement, lacking the comprehensive, integrated mechanisms needed to balance global constraints and fairness considerations effectively. As a result, there is a pressing need for frameworks that holistically incorporate constraint satisfaction, equitable resource distribution, and adaptive multi-agent coordination.% \cite{ }.

\subsection{Contributions}
This paper presents a \emph{Constrained Hierarchical Multiagent Reinforcement Learning (CH-MARL)} framework that addresses these multifaceted requirements within sustainable maritime logistics. Our main contributions are: 

\begin{enumerate}
    \item \textbf{Dynamic constraint-enforcement layer:} 
    We propose real-time mechanisms to ensure compliance with global emission caps, thereby reinforcing sustainability goals throughout the operational horizon. 
    
    \item \textbf{Fairness-aware reward shaping:} 
    We integrate fairness metrics into the reward function, promoting equitable cost and resource sharing among multiple stakeholders. %\cite{hu2020fairness}. 
    
    \item \textbf{Validation in a maritime digital twin:} 
    We demonstrate significant enhancements in efficiency and emissions reduction within a digital twin environment for maritime logistics, highlighting the practicality and effectiveness of our proposed framework.% \cite{shi2020multi}. 
\end{enumerate}

By addressing operational efficiency, environmental sustainability, and stakeholder equity, \textit{CH-MARL} extends the capabilities of existing MARL models and paves the way for robust, scalable solutions in real-world maritime settings.

\subsection{Paper Organization}
The remainder of this paper is organized as follows: Section \ref{background_related_work} provides a detailed review of the background and related work, highlighting advancements in Multiagent Reinforcement Learning (MARL), Constrained Reinforcement Learning (CRL), and their applications in maritime logistics. Section \ref{sec:methodology} presents the proposed CH-MARL framework, detailing its architecture, methodology, and key components. Section \ref{sec:experiments} outlines the experimental setup, including the digital twin environment and evaluation metrics. Results and discussions are presented in Section \ref{sec:experiments}, followed by conclusions and future work in Section \ref{sec:conclusion}.

\section{Background and Related Work} \label{background_related_work}

The rapid advancements in artificial intelligence have paved the way for the development of multiagent systems, where multiple autonomous entities interact within shared environments to achieve individual or collective goals. Multiagent Reinforcement Learning (MARL) has emerged as a critical area of research within AI, addressing the unique challenges posed by these systems. These challenges include coordinating behaviors in dynamic and uncertain environments, managing conflicts between agents with competing objectives, and ensuring the scalability of solutions as the number of agents increases. MARL is particularly relevant in domains where complex decision-making and adaptability are required, such as robotics, logistics, and industrial automation.

In parallel, there has been growing interest in Constrained Reinforcement Learning (CRL), which extends traditional reinforcement learning to incorporate safety and performance constraints. This paradigm is crucial for ensuring that AI systems operate within acceptable boundaries, especially in high-stakes applications such as autonomous driving, healthcare, and energy management. Integrating constraints into MARL adds another layer of complexity, requiring innovative approaches to balance local and global requirements across multiple agents.

The maritime and logistics industries stand out as prominent domains where MARL and CRL have significant potential. As global trade and transportation continue to grow, optimizing operations in these sectors has become increasingly critical for reducing costs, improving efficiency, and minimizing environmental impact. However, the dynamic and interconnected nature of maritime logistics presents unique challenges, including the need for hierarchical decision-making and adherence to international regulations on emissions and sustainability.

Fairness in multiagent systems is another critical aspect that has garnered attention, particularly in scenarios where resource allocation or policy decisions affect diverse stakeholders. Ensuring equitable outcomes is essential for fostering cooperation and maintaining trust among participants. This is especially relevant in industrial contexts, where imbalanced policies can disadvantage smaller players, leading to inefficiencies and conflicts.

By exploring the intersections of MARL, CRL, maritime logistics, and fairness, this work aims to address key research gaps and advance the state of the art in these fields. The following sections provide a comprehensive review of existing literature and identify opportunities for future research.

\subsection{Multiagent Reinforcement Learning}
Multiagent Reinforcement Learning (MARL) has gained increasing prominence as a method to coordinate autonomous agents in complex, dynamic environments. MARL settings can be broadly categorized into cooperative and competitive domains. In cooperative scenarios, multiple agents aim to maximize a shared objective, exemplified by tasks such as cooperative robotics or emergency response \cite{panait2005cooperative,foerster2018learning}. In contrast, competitive MARL addresses settings where agents have conflicting or opposing goals, including zero-sum games and competitive market models \cite{lowe2017multi,silver2017mastering}.

One core challenge in MARL involves \emph{scalability}: as the number of agents increases, the growth of state and action spaces becomes exponential, making learning computationally intractable \cite{hernandez2019survey}. Techniques such as decentralized training with centralized execution (DTCE) \cite{oliehoek2008optimal} and parameter sharing \cite{gupta2017cooperative} have partially mitigated this issue, although these methods often assume relatively homogeneous agent types. Additionally, agents typically operate under \emph{partial observability}, having access only to localized information \cite{foerster2016learning}, which can impede learning. Recent approaches that incorporate recurrent architectures \cite{hausknecht2015deep} and communication protocols \cite{sukhbaatar2016learning} show promise, but also increase computational overhead. Finally, \emph{non-stationarity} arises when multiple agents learn in parallel, making convergence to stable policies or equilibria more challenging \cite{zhang2021multi}. Mechanisms such as opponent modeling \cite{albrecht2018autonomous} and equilibrium learning \cite{shou2022multi} are active research areas aimed at alleviating these challenges.

\subsection{Constrained Reinforcement Learning}
Constrained Reinforcement Learning (CRL) introduces explicit constraints—often reflecting safety or performance thresholds—into the learning process. Early work employed Lagrangian relaxation \cite{altman1999constrained}, translating constrained optimization into unconstrained forms through penalty functions. More recent research on Constrained Policy Optimization (CPO) \cite{achiam2017constrained} uses trust-region methods to ensure constraints remain satisfied, offering theoretical performance guarantees at the cost of higher computational demands.

Alternative methods, such as primal-dual optimization and constraint sampling, propose additional ways of managing constraints. However, these techniques tend to struggle with scalability in multiagent contexts. When applied to MARL, the necessity of balancing both individual-agent constraints and system-wide limitations complicates policy design. Further, real-time constraints in large-scale applications—such as traffic networks or power grids—necessitate efficient CRL approaches that can handle multiagent interactions.

\subsection{MARL for Maritime \& Logistics}
The maritime sector has begun exploring MARL to optimize factors such as vessel routing, port operations, and scheduling. For instance, vessel routing studies show MARL can simultaneously reduce fuel consumption and improve scheduling efficiency. Meanwhile, research on port operations applies MARL to berth allocation, crane dispatch, and container handling.

Despite notable achievements, existing approaches rarely incorporate global emission caps enforced by international regulations like those of the IMO \cite{imo2020}. Fairness considerations are also overlooked, potentially disadvantaging smaller shipping lines. While hierarchical coordination can help manage decisions across different levels (e.g., vessel-level, fleet-level), such methods remain relatively underexplored in the maritime domain. These gaps highlight the potential for a more robust and integrated MARL framework that can reconcile emission constraints and fairness goals with overarching operational objectives.

\subsection{Fairness in Multiagent Systems}
Fairness is crucial in multiagent resource-allocation problems, as policy decisions may disproportionately affect certain stakeholders. Foundational concepts like max-min fairness, envy-freeness, and Gini coefficients offer theoretical bases for diagnosing and mitigating inequality \cite{nash1950bargaining,lipton2004approximately}. In practice, such metrics have been employed in traffic management \cite{aloor2024cooperation} and cloud computing, among other areas.

In industrial logistics, smaller firms risk systematic disadvantage when competition arises for shared resources,e.g., docking berths or emission budgets. Ensuring equitable outcomes is not only a moral imperative but also fosters cooperation, enhancing the real-world viability of multiagent solutions. Nonetheless, fair allocation is challenging to implement in MARL, where efficiency and equity often compete \cite{rawls1971theory}. More research is needed to integrate fairness constraints effectively in systems that also target high performance.

%Collectively, t
The literature indicates a strong need for hierarchical methods that incorporate both global constraints and fairness into MARL, especially in sustainable maritime logistics. Existing work offers only partial solutions, typically addressing isolated aspects like constrained optimization or static fairness scenarios. By bridging these gaps, future MARL frameworks can more effectively guide sustainable decision-making and policy enforcement in real-world maritime contexts.

\section{Problem Formulation}
\label{sec:problem_formulation}

In this section, we define the maritime logistics problem targeted by our 
hierarchical multi-agent framework. We first describe the overall system, 
including ports, vessels, and dynamic factors (Sec.~\ref{subsec:system-description}), 
followed by an explanation of agent roles in a hierarchical setting 
(Sec.~\ref{subsec:agents-hierarchies}). We then detail the formal state, 
action, and reward structures (Sec.~\ref{subsec:state-action-reward}), 
illustrate the global constraints that must be respected 
(Sec.~\ref{subsec:global-constraints}), and motivate the fairness objectives 
that ensure equitable resource distribution among agents 
(Sec.~\ref{subsec:fairness-objectives}). Finally, we summarize the domain 
assumptions and simplifications adopted for tractability 
(Sec.~\ref{subsec:assumptions-simplifications}).

\subsection{System Description}
\label{subsec:system-description}

Maritime logistics networks are composed of multiple \emph{ports} and 
\emph{vessels} transporting goods along specified routes. Each port 
provides critical services---such as berths, cranes, and cargo-handling 
facilities---to the vessels arriving at scheduled or unscheduled intervals. 
In our setting, we assume a network $\mathcal{G} = \{\mathcal{P}, \mathcal{R}\}$, 
where $\mathcal{P}$ is the set of ports and $\mathcal{R}$ is the set of 
possible sea routes between these ports. 

A typical planning or operational horizon (e.g., daily or weekly) involves 
\emph{multiple vessels} traveling between ports under dynamically changing 
factors:
\begin{itemize}[noitemsep, topsep=0pt]
    \item \textbf{Weather conditions:} Storms, wind, and waves can affect fuel consumption and vessel speed.
    \item \textbf{Port congestion:} High traffic at a port may lead to queuing delays and increased emissions due to idling.
    \item \textbf{Route variations:} Different sea lanes or alternative channels can exist, each with distinct distance, fuel cost, and traffic patterns.
    \item \textbf{Mechanical uncertainties:} Vessels can experience engine wear, maintenance needs, or unexpected failures that alter operational performance.
\end{itemize}
These \emph{dynamic factors} introduce \emph{stochasticity} and 
\emph{partial observability} into the system, meaning no single entity has a perfectly accurate or complete snapshot of all relevant parameters in real time.

\subsection{Agents and Hierarchical Roles}
\label{subsec:agents-hierarchies}

We model the network using a set of autonomous, decision-making \emph{agents}. Following a \emph{hierarchical} paradigm, these agents are divided into two primary layers:

\paragraph{High-Level Agents (Strategic Layer).} 
These agents operate at a coarser timescale (e.g., hours or days) and have a broader view of the logistics network. Their responsibilities include:
\begin{itemize}[noitemsep, topsep=0pt]
    \item \textbf{Route Planning:} Selecting high-level paths through the port network for each vessel, considering distance, congestion, and overall fleet coordination.
    \item \textbf{Emission Budgeting:} Allocating emission allowances across multiple vessels or routes based on predicted fuel consumption and regulatory thresholds.
    \item \textbf{Scheduling Coordination:} Coordinating vessel 
    departure/arrival times at ports to minimize queueing and turnaround delays.
\end{itemize}

\paragraph{Low-Level Agents (Operational Layer).}
At finer granularity (e.g., every few minutes or seconds), low-level agents make local decisions consistent with the strategic directives. Examples include:
\begin{itemize}[noitemsep, topsep=0pt]
    \item \textbf{Speed Control:} Dynamically adjusting vessel speed to meet emission targets, fuel constraints, or arrival deadlines.
    \item \textbf{Berthing and Cargo Handling:} Managing docking assignments and crane usage based on immediate port conditions, workforce availability, and mechanical constraints.
    \item \textbf{On-Board Resource Management:} Monitoring engine performance, fuel flow rates, and battery/hybrid systems (if 
    available).
\end{itemize}
This division of roles allows each layer to focus on decision-making at its own operational scope, thereby increasing \emph{scalability} and \emph{adaptability} in large-scale maritime scenarios.

\subsection{State, Action, and Reward Structures}
\label{subsec:state-action-reward}

\paragraph{State Space.}
Let $s_t \in \mathcal{S}$ denote the global state of the maritime system 
at time $t$, which may include:
\begin{itemize}[noitemsep, topsep=0pt]
    \item \emph{Vessel states:} Locations, velocities, fuel levels, mechanical health indicators.
    \item \emph{Port states:} Queue lengths, berth occupancy, crane availability, operational schedules.
    \item \emph{Environmental data:} Partial or noisy weather forecasts, wave heights, visibility.
    \item \emph{Historical usage:} Aggregated emissions or fuel usage up to time $t$, to track proximity to global caps.
\end{itemize}
In practice, each agent observes only a subset of the global state (i.e., $o^i_t \subseteq s_t$), reflecting \emph{partial observability}.

\paragraph{Action Space.}
Each agent $i$ selects an action $a^i_t$ from its individual action space $\mathcal{A}_i$. Examples include:
\begin{itemize}[noitemsep, topsep=0pt]
    \item \emph{High-Level Actions:} Assign route $r_k$ to vessel $v_j$, issue emission allowance $B_j$ for a time window $\Delta t$.
    \item \emph{Low-Level Actions:} Increase or decrease vessel speed, request a specific berth at port $p_m$, dispatch cargo-handling equipment.
\end{itemize}
The \emph{hierarchical} structure means that high-level actions can constrain the choices available to low-level agents (e.g., limiting feasible speed ranges or specifying which port to approach next).

\paragraph{Reward Structure.}
We define a \emph{weighted} or \emph{composite} reward to balance 
multiple objectives:
\[
    r^i_t \;=\; \underbrace{R_{\text{cost}}(a^i_t)}_{\text{fuel, time}} 
    \;+\; \underbrace{R_{\text{emission}}(a^i_t)}_{\text{GHG reductions}} 
    \;+\; \underbrace{R_{\text{fair}}(a^i_t)}_{\text{equity term}}.
\]
Concretely, $R_{\text{cost}}$ might be negatively proportional to fuel consumption or berth fees; $R_{\text{emission}}$ rewards strategies that lower CO$_2$ outputs; and $R_{\text{fair}}$ enforces equity constraints (see Sec.~\ref{subsec:fairness-objectives}). Depending on the \emph{global} or \emph{local} nature of each objective, these terms may be shared or agent-specific.

\subsection{Global Constraints}
\label{subsec:global-constraints}

Aside from local reward optimization, maritime operators must comply with strict \emph{global constraints}, such as:
\begin{itemize}[noitemsep, topsep=0pt]
    \item \textbf{Emission Caps:} A regulatory or self-imposed limit on total CO$_2$ or other pollutants (e.g., NO$_x$) for each route, port area, or entire shipping alliance over a fixed horizon.
    \item \textbf{Port Resource Capacity:} Physical constraints on the number of vessels that can be serviced at a port simultaneously (limited berths, crane scheduling, labor availability).
\end{itemize}
In a multi-agent context, these constraints are \emph{shared} by all agents; thus, a single vessel’s high fuel usage or a congested port schedule can push the system over its limit. Our framework addresses these concerns via \emph{primal-dual} penalty mechanisms or other constraint-enforcement algorithms, ensuring real-time tracking of resource usage and immediate penalization for exceeding thresholds.

\subsection{Fairness Objectives}
\label{subsec:fairness-objectives}

Due to the heterogeneous nature of vessels and shipping lines, 
\emph{fairness} emerges as a critical objective. Without explicit 
fairness considerations, certain agents (e.g., smaller operators) may routinely end up with unfavorable slots or higher operational costs, undermining cooperation and real-world adoption. We thus define a \emph{fairness metric}, $\mathcal{F}(\{c_i\})$, where $c_i$ represents the cumulative cost or burden accrued by agent $i$. Possible metrics include:
\begin{itemize}[noitemsep, topsep=0pt]
    \item \emph{Gini coefficient:} Measures inequality in the distribution of $c_i$ values.
    \item \emph{Max-min fairness:} Minimizes the maximum deviation between agents’ costs.
    \item \emph{Envy-freeness:} Ensures no agent would prefer another agent’s allocation of resources to its own.
\end{itemize}
We incorporate $\mathcal{F}$ into agent rewards or treat it as an 
\emph{auxiliary constraint} to ensure that operational burdens and 
benefits are equitably distributed across the fleet.

\subsection{Assumptions and Simplifications}
\label{subsec:assumptions-simplifications}

To manage complexity, we adopt the following assumptions and 
simplifications:
\begin{itemize}[noitemsep, topsep=0pt]
    \item \textbf{Discrete Time Steps:} Both high-level and low-level decisions occur at discrete intervals (e.g., 15-minute or hourly increments), approximating continuous processes.
    \item \textbf{Simplified Weather Model:} While weather is stochastic, we assume a finite set of forecast scenarios (e.g., calm, moderate, storm). Agents receive partial updates regarding the current scenario’s likelihood at each step.
    \item \textbf{Limited Mechanical Failures:} We consider mechanical uncertainties (engine wear, malfunctions) in a simplified manner, modeling them as random events with fixed probability distributions.
    \item \textbf{Agent Collaboration vs.~Competition:} We focus primarily on \emph{cooperative or semi-cooperative} settings, acknowledging that in real-world shipping alliances, competition can also shape decision incentives. If needed, competition can be introduced by adjusting reward functions or agent objectives.
\end{itemize}

For tractability, we adopt a discrete-time model with a finite set of weather scenarios, each occurring with known probability. Mechanical failures are similarly modeled as random events with specified likelihoods. We primarily consider cooperative or semi-cooperative scenarios where agents aim to achieve shared sustainability and operational goals, although competitive dynamics can be introduced by altering reward structures. Under these assumptions, the problem becomes a constrained hierarchical MARL setting: agents in each layer learn policies $\pi^i$ that optimize local rewards while adhering to global constraints and fairness criteria, thereby aiming to achieve robust, equitable, and environmentally sound maritime logistics.

%\paragraph{Relevance to Methodology.}
By formally describing the maritime environment, agents, actions, rewards, constraints, and fairness goals, we establish a clear foundation for the \emph{Methodology} (Sec.~\ref{sec:methodology}). In particular, the \emph{state-action-reward} definitions guide how each hierarchical agent learns policies, while the \emph{global constraints} and \emph{fairness objectives} inform our choice of constrained and fairness-aware reinforcement learning algorithms.

\section{Methodology} \label{sec:methodology}
\subsection{Theoretical Foundations}
\label{sec:theoretical-foundations}

%\subsubsection{Multi-Agent Reinforcement Learning (MARL)}
Reinforcement learning (RL) is a paradigm in which an agent interacts with an environment to maximize long-term rewards. In MultiAgent RL (MARL), multiple agents concurrently learn and act within a shared environment. A standard model for an $N$-agent system is 
$
(\mathcal{S}, \{\mathcal{A}_i\}_{i=1}^N, P, \{\mathcal{R}_i\}_{i=1}^N, \gamma),
$ 
where $\mathcal{S}$ is the state space, $\mathcal{A}_i$ is agent $i$'s action set, $P(\cdot \mid s,a)$ specifies the transition dynamics, $\mathcal{R}_i$ is the reward function for agent $i$, and $0 < \gamma \le 1$ is the discount factor.

In many cooperative or semi-cooperative domains—such as maritime logistics—agents must jointly optimize system-level objectives like energy savings or on-time deliveries. However, classical MARL algorithms often focus on maximizing local or team rewards, overlooking important aspects such as global constraints (e.g., carbon budgets) or \emph{fairness} in resource distribution. These omissions motivate our extension to constrained and fairness-aware MARL approaches.

%\subsubsection{Constrained Reinforcement Learning}
Although single-agent RL traditionally optimizes rewards without explicit constraints, real-world settings frequently require compliance with safety, budget, or emission regulations. \emph{Constrained RL} addresses such scenarios by incorporating resource or policy constraints into the learning objective. A common method uses a primal-dual (Lagrangian) framework, assigning a Lagrange multiplier to each constraint. When the agent's actions exceed a threshold (e.g., an emission cap), the associated multiplier grows, penalizing future violations.

Adapting constrained RL to multi-agent settings introduces additional complexity, especially when violations stem from aggregate actions rather than a single agent. We therefore implement a mechanism that dynamically coordinates individual agent policies to keep a shared constraint (like total greenhouse gas emissions) within permissible limits, even when agents only observe partial states.

%\paragraph{Constraint Satisfaction in CH-MARL}
We capture constraint satisfaction via the Lagrangian formulation: 
%
%\begin{defn}[Lagrangian Function]
%\label{defn:lagrangian}
$%\[
\mathcal{L}(\pi, \lambda)
\;=\;
\mathbb{E}[R]
\;-\;
\lambda\,(\mathbb{E}[C] - \kappa),
$ %\]
where $R$ is the reward function, $C$ is the constraint function, $\kappa$ is the allowable threshold, and $\lambda$ is the dual variable. 
%\end{defn}
%
This function encodes a balance between maximizing expected reward and penalizing constraint violations. Under smoothness assumptions, standard gradient-based updates can converge to a feasible policy \(\pi^*\). 

\begin{prop}[Convergence to a Constraint-Satisfying Policy]
\label{prop:1}
If \(\mathcal{L}(\pi,\lambda)\) is continuously differentiable and \(\mathbb{E}[C]\) is convex w.r.t.\ \(\pi\), then a gradient-based optimization over \((\pi,\lambda)\) converges to a policy \(\pi^*\) such that \(\mathbb{E}[C]\le \kappa\).
\end{prop}

\begin{proof}
By iteratively adjusting \(\pi\) in the direction of \(\nabla_\pi \mathcal{L}\) and \(\lambda\) in the opposite direction of \(\nabla_\lambda \mathcal{L}\), we locate a saddle point \((\pi^*, \lambda^*)\). The saddle point satisfies \(\nabla_\pi\mathcal{L}=\nabla_\lambda\mathcal{L}=0\), thus enforcing \(\mathbb{E}[C]\le\kappa\).
\end{proof}

\begin{prop}[Bounded Constraint Violations]
A dynamic enforcement layer ensures \(\|\mathbb{E}[C]-\kappa\|\le \epsilon\), where \(\epsilon\) depends on the learning rates and problem smoothness.
\end{prop}

\begin{proof}
Updating each agent's reward with a penalty term \(-\lambda\,(\mathbb{E}[C]-\kappa)\) and incrementing \(\lambda\) whenever constraints are exceeded leads to bounded overshoot. The iterative approach converges so violations remain below a chosen tolerance \(\epsilon\).
\end{proof}

%\subsubsection{Hierarchical Reinforcement Learning (HRL)} \label{sec:hrl}
Standard MARL frameworks can become intractable if every agent must learn fine-grained control in large or continuous state-action spaces. \emph{Hierarchical RL} decomposes the learning problem into multiple layers, with high-level (strategic) policies specifying abstract actions (e.g., route selection or resource budgets), and low-level (operational) policies refining these macro-actions at finer time scales.

\paragraph{Hierarchical Policy Convergence}
CH-MARL leverages hierarchical RL to contain complexity. High-level decisions are modeled as a Constrained Markov Decision Process (CMDP), while low-level decisions act within constraints set by the high-level layer. 

\begin{defn}[Constrained Markov Decision Process (CMDP)]
\label{defn:cmdp}
A CMDP is a tuple 
\(
\langle \mathcal{S}, \mathcal{A}, \mathcal{P}, \mathcal{R}, \mathcal{C}, \kappa\rangle
\)
where \(\mathcal{S}\) and \(\mathcal{A}\) are state and action spaces, \(\mathcal{P}\) is the transition model, \(\mathcal{R}\) the reward, \(\mathcal{C}\) the constraint function, and \(\kappa\) the threshold. The objective is to maximize \(\mathbb{E}[\sum_t\mathcal{R}(s_t,a_t)]\) subject to \(\mathbb{E}[\sum_t\mathcal{C}(s_t,a_t)]\le\kappa\).
\end{defn}

\begin{prop}[Hierarchical Policy Convergence]
\label{prop:2}
Let \(\pi^H\) denote the high-level policy and \(\pi^L\) the low-level policy. If both \(\pi^H\) and \(\pi^L\) are updated via gradient-based algorithms that incorporate constraint penalties, they converge to locally optimal solutions under bounded constraints.
\end{prop}

\begin{proof}
We treat high-level decisions as a CMDP, updating \(\pi^H\) and the dual variable \(\lambda\) to satisfy global constraints. Within each high-level decision interval, \(\pi^L\) maximizes local rewards. Combining these levels in a nested, iterative scheme yields a joint saddle point, ensuring global feasibility and locally optimal hierarchical decisions.
\end{proof}

%\subsubsection{Fairness in Multi-Agent Systems}
Fairness is critical in scenarios involving multiple stakeholders with heterogeneous capabilities or resources. Without explicit fairness mechanisms, solutions can emerge that optimize aggregate metrics but consistently disadvantage smaller agents. Metrics like Gini coefficients, max-min fairness, or envy-free allocations can be integrated into the learning objective or enforced through constraints.

\paragraph{Fairness Guarantees}
We embed fairness into CH-MARL by introducing an adjustment term that penalizes large disparities in individual agent burdens, such as fuel usage or scheduling delays.

\begin{prop}[Fairness Metric Guarantees]
\label{prop:3}
If the max-min fairness criterion \( F=\min_i\bigl(\frac{R_i}{R_{\text{optimal}}}\bigr) \) is integrated into the reward structure, then the system can guarantee \(F \ge \delta\) for a chosen threshold \(\delta\), provided the fairness penalty is scaled appropriately.
\end{prop}

\begin{proof}
Each agent $i$ has an adjusted reward \(R_i'\!= R_i - \beta\,\Delta_i\), where \(\Delta_i\) captures deviations from fair allocation (e.g., differences from the best-performing agents). Iteratively reducing \(\Delta_i\) via gradient-based updates keeps \(F\) above \(\delta\). Appropriate tuning of \(\beta\) ensures sustained equity across agents.
\end{proof}

Collectively, these propositions establish the theoretical foundations for a multi-agent framework capable of hierarchical decision-making, real-time constraint satisfaction, and equitable resource distribution.

\subsection{CH-MARL Framework}\label{sec:ch-marl}
We now synthesize the above elements into \emph{Constrained Hierarchical Multi-Agent Reinforcement Learning}, illustrated in Figure~\ref{fig:ch-marl}. High-level (strategic) agents operate on extended timescales to make macro-decisions, while low-level (operational) agents refine these directives at fine-grained intervals. A primal-dual constraint layer enforces global limits (e.g., emission caps), and a fairness module adjusts the reward signals to prevent excessive burdens on any single agent.

\begin{figure}[!ht]
    \centering
    \includegraphics[width=.5\linewidth]{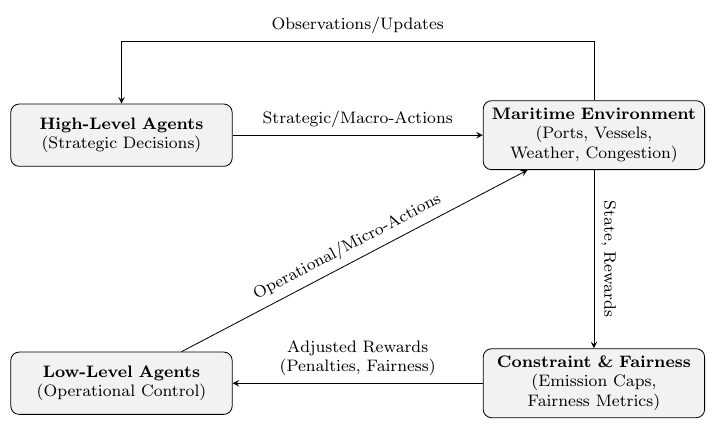}
    \caption{Conceptual overview of CH-MARL, highlighting the division between high-level (strategic) and low-level (operational) agents, along with the primal-dual and fairness modules that modulate rewards.}
    \label{fig:ch-marl}
\end{figure}

\subsubsection{Hierarchical Agent Architecture:}
High-level agents address route planning, emission allocations, and schedule coordination, informed by partial yet broader observations (e.g., aggregated port loads or regulatory updates). Low-level agents then execute localized decisions, such as speed adjustments or berth assignments, in alignment with high-level directives. By offloading operational details to low-level agents, the overall system reduces complexity and accelerates learning in large-scale domains.

\subsubsection{Primal-Dual Constraint Enforcement:}
\label{sec:constraint-enforcement}
A global emission budget $C_{\max}$ enforces sustainable operations. At each step, cumulative emissions are measured, and any overshoot updates the Lagrange multiplier $\lambda$. All agents experience a penalty proportional to $\lambda$ and their incremental emissions, driving them toward compliance. This mechanism efficiently handles collective constraints under partial observability, since agents learn to avoid globally penalized actions.

\subsubsection{Fairness-Aware Reward Shaping:}
\label{sec:fairness}
In addition to meeting sustainability goals, maritime systems must ensure that smaller or less-resourced stakeholders are not systematically disadvantaged. We incorporate a fairness function \(\mathcal{F}(\{c_i\})\), where \(c_i\) denotes agent $i$'s cumulative burden (e.g., fuel spent or waiting time). A fairness term \(\Delta_{\text{fair}}\) is added to or subtracted from each agent's reward to encourage parity, thereby balancing efficiency with equitable outcomes.

\subsubsection{Integrated Workflow of CH-MARL:}
The full CH-MARL process proceeds in the following loop:
1. \text{High-Level Action Selection:} Each strategic agent samples a macro-action (e.g., route choice, emission allocation) at a coarse timescale.
2. \text{Low-Level Refinement:} Operational agents refine these high-level directives into granular actions (e.g., throttle settings, berth scheduling) using local observations.
3. \text{Environment Step:} The environment transitions to a new state and returns partial rewards, emission data, and fairness-relevant statistics.
4. \text{Constraint \& Fairness Updates:} A primal-dual module updates the Lagrange multiplier $\lambda$ if emissions exceed $C_{\max}$, and a fairness module adjusts rewards based on the observed resource distribution.
5. \text{Policy Updates:} Each agent—both strategic and operational—uses an RL algorithm (e.g., actor-critic or Q-learning) to update its policy according to the adjusted reward signals. %
Algorithm~\ref{alg:ch-marl-combined} provides an overview of the CH-MARL training loop. It illustrates how high-level (strategic) and low-level (operational) agents make decisions, how primal-dual updates enforce global emission constraints, and how fairness terms adjust individual rewards. %The procedure covers each stage, from action selection and environment transitions to reward shaping, policy updates, and record logging.
After sufficient training episodes, the hierarchical policies converge to strategies that simultaneously respect global emission constraints, maintain operational efficiency, and promote fair outcomes among a diverse set of agents.

\setlength{\textfloatsep}{3pt}% Remove \textfloatsep
\begin{algorithm}[!ht]
\caption{\textbf{CH-MARL Framework}}% with Primal-Dual Constraint \& Fairness Enforcement}}
\label{alg:ch-marl-combined}
\SetAlgoLined
\DontPrintSemicolon

\KwIn{
\begin{itemize}[noitemsep, topsep=0pt]
  \item Environment $E$ (MDP or partially observable)
  \item High-level agent set $\mathcal{H} = \{h_1, \dots\}$
  \item Low-level agent set $\mathcal{L} = \{\ell_1, \dots\}$
  \item Emission cap $C_{\max}$, fairness function $\mathcal{F}$
  \item Learning hyperparameters (rates, discount $\gamma$, etc.)
  \item Number of episodes $N$, time horizon $T$
\end{itemize}
}
\KwOut{
\begin{itemize}[noitemsep, topsep=0pt]
  \item Learned hierarchical policies $\{\pi^H_i\}$ and $\{\pi^L_j\}$
  \item Updated dual variable $\lambda$ for emission constraints
\end{itemize}
}

\textbf{Initialize:} 
  Policies $\pi^H_i$ for each high-level agent $h_i \in \mathcal{H}$, 
  and $\pi^L_j$ for each low-level agent $\ell_j \in \mathcal{L}$;
  $\lambda \leftarrow 0$ (dual var. for emission cap);
  replay buffers $\mathcal{D}_a$ \\%(if needed) \\
 
\For{$\text{episode} \;=\;1 \to N$}{
   \textbf{Observe} initial state $s_0$; \\
   \For{$t=0 \to T-1$}{
      %\emph{\# 1. High-Level Decisions} \\
      \ForEach{$h_i \in \mathcal{H}$}{
         $o_i \leftarrow \text{getObservation}(s_t, h_i)$; \\
         $a^H_i \leftarrow \text{sampleAction}(\pi^H_i, o_i)$;
      }

     % \emph{\# 2. Low-Level Decisions} \\
      \ForEach{$\ell_j \in \mathcal{L}$}{
         $o_j \leftarrow \text{getObservation}(s_t, \ell_j)$; \\
         $a^L_j \leftarrow \text{sampleAction}(\pi^L_j, o_j, \{a^H_i\})$;
      }

      %\emph{\# 3. Environment Step} \\
      $A_t \;=\;\{a^H_i\} \,\cup\, \{a^L_j\}$; \\
      $(s_{t+1},\,r_t,\,\textit{metrics}) \leftarrow E.\text{step}(A_t)$;

      %\emph{\# 4. Constraint \& Fairness Enforcement (Merged Steps)} \\
      %\texttt{// Emission and fairness checks} \\
      $e_t \leftarrow \text{getEmissions}(\textit{metrics})$; \\
      $\text{cumEmissions} \leftarrow \text{updateCumulative}(e_t)$; \\
      \If{ $\text{cumEmissions} > C_{\max}$ }{
         $\phi \leftarrow \text{cumEmissions} - C_{\max}$; \\
         $\lambda \leftarrow \lambda + \alpha_\lambda \cdot \phi$; \\
      }
      \ForEach{agent $i$ (both high-level and low-level)}{
         %\texttt{// Emission penalty term} \\
         $r_{t}^{\text{constraint},i} \leftarrow -\,\lambda \times e_t^{(i)}$; \\
         %\texttt{// Fairness adjustment} \\
         $\text{cost}_i \leftarrow \text{getCurrentBurden}(i)$; \\
         $\delta_{\text{fair},i} \leftarrow \text{computeFairnessTerm}(\text{cost}_i,\mathcal{F})$; \\
         %\texttt{// Adjust final reward} \\
         $r_t^{\text{adj},i} \leftarrow r_t^i + r_t^{\text{constraint},i} + \delta_{\text{fair},i}$;
      }

      %\emph{\# 5. Policy Updates} \\ %For some methods (like PPO), you may not store transitions in a replay buffer but rather in on-policy rollouts. You can customize step #5 accordingly.
      \ForEach{agent $a \in (\mathcal{H} \cup \mathcal{L})$}{
         $\mathcal{D}_a.\text{store}(s_t,A_t,r_t^{\text{adj}},s_{t+1})$; \\
         $\text{updatePolicy}(\pi_a,\mathcal{D}_a,\lambda)$;
      }
      $s_t \leftarrow s_{t+1}$;
   }
}
\Return $(\{\pi^H_i\}, \{\pi^L_j\}, \lambda)$
\end{algorithm}

\subsection{Complexity Analysis and Practical Scalability}
\paragraph{Time Complexity}
The hierarchical design partitions policy learning into a strategic layer and an operational layer. For each episode of length $T$, high-level policy updates scale with $O(T\,|A^H|\,|S|)$, while the $n$ low-level policies add $O(n\,T\,|A^L|)$. Thus, total cost is 
$
O\bigl(T\,|A^H|\,|S|\;+\;n\,T\,|A^L|\bigr).
$ 
Practical frameworks reduce these costs via state abstraction, concurrent policy updates, or approximate solvers.

\paragraph{Space Complexity:}
In addition to storing policies (potentially neural networks) for both hierarchical layers, replay buffers or model-based structures require memory proportional to the size of $(\mathcal{S},\mathcal{A})$. For CMDPs, transition tables can add complexity up to $O(|S|^2\,|A^H|)$. 

\paragraph{Scalability:}
CH-MARL handles moderate-scale domains effectively, yet large state-action spaces require techniques, e.g., aggregating similar states to reduce dimensionality, simultaneously update multiple low-level policies on multi-core/distributed architectures, and pruning action spaces or using deep neural function approximators for high-dimensional inputs. 
These strategies ensure it remains tractable and responsive in complex and dynamic settings.

CH-MARL harnesses hierarchical RL to divide global tasks and partial observability among multiple agents, employs primal-dual updates to enforce dynamic constraints such as emission caps, and embeds fairness objectives to protect smaller or more vulnerable stakeholders. The synergy of these components offers a scalable pathway toward optimizing large, industrial-grade systems that must balance efficiency, sustainability, and equity. %In Section~\ref{sec:experimental-setup}, we describe the practical implementation details of CH-MARL, followed by an evaluation on a maritime logistics digital twin in Section~\ref{sec:experiments}, where we demonstrate how these theoretical advantages translate into significant improvements in emission reduction, resource allocation fairness, and operational throughput.

\section{Experimental Setup}
\label{sec:experimental-setup}

\subsection{Digital Twin Environment} 
\label{subsec:digital-twin}

We develop a digital twin that closely emulates real-world maritime logistics operations on a smaller, controlled scale. The simulation environment implements four primary components:
a) {Shipping Routes:}
The network is modeled as directed edges connecting ports (nodes), each with an associated distance, typical journey duration, and potential weather disruptions.
b) {Vessel Movements:}
Each vessel agent operates at discrete time intervals (e.g., hourly), updating its position and fuel usage. Agents can alter speeds or perform minor route deviations within prescribed bounds, mirroring practical operational constraints.
c) {Port Operations:}
Ports have finite berth capacities, modeled by limiting the number of vessels that can be simultaneously served. If more vessels arrive at a port than available berths, queues form, thereby increasing waiting times and idle fuel consumption. Each vessel accumulates queue time (in hours) whenever it is forced to wait.
d) {Weather Stochasticity:}
Wind speed, storms, and other meteorological factors are sampled from historical or synthetic distributions. These conditions impact vessel navigation (e.g., speed reductions) and fuel burn rates, introducing real-world uncertainty into the simulation.

For fidelity, the digital twin incorporates empirical data on port statistics (e.g., berth occupancy, crane throughput), vessel characteristics (e.g., hull design, engine power, fuel-consumption curves), and weather records (e.g., wind speed, wave height). These elements collectively provide a realistic yet tractable environment for evaluating multi-agent policies.

\subsection{Key Performance Indicators (KPIs)}
\label{subsec:kpis}

We track the following KPIs to evaluate economic, environmental, and fairness dimensions: 
a) {Energy Consumption:} The total fuel usage aggregated over all vessels. Lower consumption indicates improved efficiency and reduced operational costs. 
b) {Total Emissions:} Represents greenhouse gas outputs (e.g., CO$_2$, NO$_x$, SO$_x$) in tons, capturing direct environmental impacts of decisions like route selection and speed control.
c) {Fairness Indices:}
Metrics such as the Gini coefficient gauge how equitably resources or costs (e.g., fuel usage) are distributed among agents. A lower Gini coefficient indicates a more balanced distribution.
d) {Operational Throughput:}
The total number of voyages completed or total cargo tonnage moved, reflecting overall system productivity within a simulation horizon.
e)  {Constraint Violation Rates:}
Monitors how frequently and by how much global constraints (e.g., emission caps, port capacity limits) are exceeded, thereby measuring compliance with regulatory requirements.
f) {Queue Times:}
Records the total waiting hours for all vessels when port capacities are exceeded, providing insight into potential delays arising from resource constraints or route congestion. 
These KPIs are recorded each episode or at periodic intervals, facilitating both fine-grained analyses (e.g., learning curves) and final policy benchmarking.

\subsection{Comparative Baselines}
\label{subsec:baselines}

To evaluate the effectiveness of \text{CH-MARL}, we compare its performance against three baselines:
(a) %Standard (Unconstrained) MARL:
A decentralized, multi-agent RL setup without explicit emission caps or fairness mechanisms. This baseline typically optimizes throughput or other local objectives but risks exceeding emission targets and privileging larger stakeholders.
(b) %Constrained Single-Agent RL:
A simplified approach that aggregates the entire system into one ``super agent'' responsible for adhering to emission limits. While this can handle small problems, it scales poorly and disregards the autonomy of individual vessels.
(c) %Hierarchical MARL without Fairness or Constraint Enforcement: 
A hierarchical method that omits global emission caps and fairness shaping. By comparing against this baseline, we isolate whether improved performance stems primarily from hierarchical decomposition or from explicitly incorporating environmental and equity considerations.

\subsection{Training and Hyperparameters}
\label{subsec:train-hyperparams}

%Each training run spans $N$ episodes, each lasting $T$ discrete time steps. For example, an episode might represent one week of simulated operations, with hourly time steps leading to $T = 168$. We generally employ actor-critic or policy gradient algorithms (e.g., PPO, A2C) with a learning rate $\alpha \in [10^{-4},10^{-3}]$. In our experiments, \texttt{Adam} is used as the optimizer, with separate actor and critic learning rates, for instance, \(\alpha_{\text{actor}} = 5 \times 10^{-4}\) and \(\alpha_{\text{critic}} = 1 \times 10^{-3}\). 

%We set the discount factor $\gamma=0.99$ to balance immediate rewards against longer-term gains. Exploration is maintained via stochastic policies with an entropy bonus in policy-gradient methods or an $\epsilon$-greedy approach in Q-based methods. Typical schedules might start with $\epsilon=0.2$ and decay it to $0.01$ over the first $500$ episodes. 

%When enforcing emission caps through primal-dual methods, the dual variable $\lambda$ is updated at each time step based on observed constraint violations. The associated learning rate $\alpha_\lambda$ often falls between $0.001$ and $0.01$. To promote equity, each agent has a fairness weight $\gamma_i=0.1$ by default, though we vary this parameter in an ablation study to examine trade-offs between fairness and efficiency.

Each training run consists of $N$ episodes, each lasting $T$ discrete time steps. For instance, one episode might simulate a week of maritime operations with hourly decision points, yielding $T = 168$. We primarily use actor-critic or policy gradient algorithms with a learning rate $\alpha \in [10^{-4},10^{-3}]$. 
In our experiments, the \texttt{Adam} optimizer is employed for both actor and critic networks, with potential hyperparameters such as \(\alpha_{\text{actor}} = 5 \times 10^{-4}\) and \(\alpha_{\text{critic}} = 1 \times 10^{-3}\). We set the discount factor $\gamma = 0.99$ to weight long-term returns. Exploration is maintained through either stochastic policies (with an entropy bonus) or an $\epsilon$-greedy approach. Typical schedules might start at $\epsilon=0.2$ and decay linearly to $0.01$ over the first $500$ episodes.

When applying emission caps through primal-dual updates, a dual variable $\lambda$ is incrementally adjusted based on constraint violations. The corresponding learning rate $\alpha_\lambda$ usually lies between $0.001$ and $0.01$. To encourage equitable allocations, each agent’s reward is penalized by a fairness term (e.g., a scaled Gini coefficient), whose weight $\gamma_i$ can be tuned (default $0.1$). Ablation studies are conducted to examine how different fairness weights affect the trade-off between equality and efficiency.

\section{Experiments and Results}
\label{sec:experiments}

We evaluate the proposed \text{CH-MARL} framework in a synthetic maritime environment featuring $8$ ports (each with limited berth capacity) and $5$ vessels. Each vessel is assigned a speed profile and cubic fuel-consumption curve typical of maritime operations. We train a single proximal policy optimization agent~\cite{schulman2017proximal} for up to $1200$ episodes, with each episode spanning $T=50$ steps, totaling $60{,}000$ time steps. Across multiple runs, we activate or deactivate the following key features: 
a) \text{Emission Caps:} A global emission threshold $C_{\max}$ is enforced. Once the system’s cumulative emissions exceed $C_{\max}$, additional penalties prompt agents to conserve fuel. 
b) \text{Fairness:} An offline penalty is computed based on the disparity in fuel usage among vessels, nudging agents to adopt more equitable cost distributions. 
c) \text{Partial Observability:} Half of the state variables are randomly masked at each step, capturing realistic uncertainty in sensor data or communications. 
d) \text{Storms:} Each time step, a 20\% chance of adverse weather reduces vessel speeds and inflates fuel consumption to simulate real-world disruptions.

We illustrate the training dynamics for four main configurations:
a) \text{Run A (Base):} No emission cap, no fairness, no storms, and full observability.
b) \text{Run B (Cap):} Employs an emission cap $C_{\max}=800$, but omits fairness and storms.
c) \text{Run C (Fair+Storms):} Excludes an emission cap but enables fairness penalties, partial observability, and storms.
d) \text{Run D (Cap+Fair+Storms):} Activates all features: $C_{\max}=800$, fairness penalties, storms, and full observability.

\begin{figure}[ht]
    \centering
    \includegraphics[width=.5\linewidth]{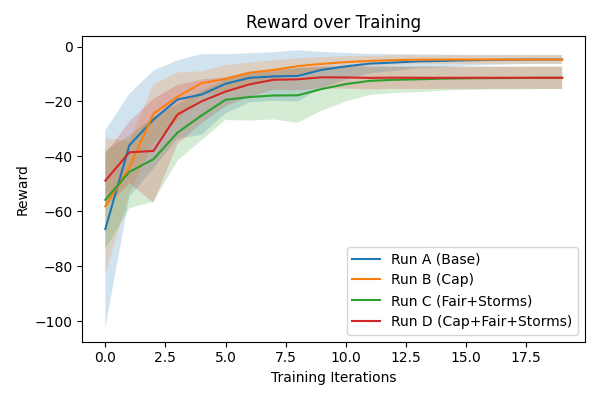}
    \caption{Reward curves (mean~$\pm$~std over three seeds) for each run. Emission caps and fairness both tend to decrease the max reward--additional constraints on vessel behaviors.}
    \label{fig:reward-curve}
\end{figure}

\begin{figure}[ht]
    \centering
    \includegraphics[width=.5\linewidth]{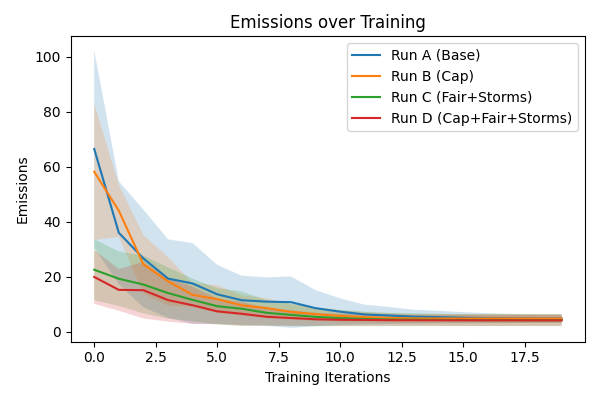}
    \caption{Emissions curves (mean~$\pm$~std) across training. Runs with active caps (Run~B, Run~D) converge to lower fuel usage, while storms and fairness can introduce variability.}
    \label{fig:emission-curve}
\end{figure}

Figures~\ref{fig:reward-curve} and \ref{fig:emission-curve} show the evolution of average reward and cumulative emissions under each configuration to highlight how each feature influences policy convergence. The addition of emission caps (Run~B, Run~D) consistently pulls emissions downward, whereas fairness and storms (Run~C, Run~D) can amplify stochasticity and reduce final rewards. 

\begin{table}[ht]
    \centering
    \caption{Final Iteration Results (Mean Across Seeds)}
    \label{tab:final-iteration-table}
    \begin{tabular}{lcc}
    \hline
    \textbf{Run} & \textbf{Reward} & \textbf{Emissions} \\
    \hline
    Run A (Base)         & -4.730401 &  4.730401  \\
    Run B (Cap)          & -4.733810  & 4.733810  \\
    Run C (Fair+Storms)  & -11.408628  & 4.099604  \\
    Run D (Cap+Fair+Storms) & -11.377978 &  4.071516  \\
    \hline
    \end{tabular}
\end{table}

Table~\ref{tab:final-iteration-table} lists each run’s mean results at the last training iteration. \emph{Run A} achieves a relatively high (less negative) reward but also incurs higher emissions. Adding an emission cap in \emph{Run B} effectively keeps emissions near $4.73$, albeit at the cost of a slightly lower reward. Meanwhile, \emph{Run C} (Fair+Storms) imposes a strong fairness penalty and faces storm disruptions, causing the reward to drop to $-11.40$ while emissions remain moderate at $4.09$. Combining all features in \emph{Run D} leads to the lowest emissions (4.07) with a mid-range reward of $-11.37$, reflecting the interplay among global constraints, partial-equity objectives, and environmental disturbances. 

These outcomes underscore the multi-objective nature of CH-MARL: emission caps enforce sustainability, fairness can induce lower but more balanced rewards, and storms introduce realistic unpredictability. Despite these constraints, all runs consistently converge on stable policies, confirming that CH-MARL maintains both scalability and adaptability when subject to diverse regulatory and environmental conditions.

\section{Conclusion and Future Work}\label{sec:conclusion}

%This paper introduced \emph{Constrained Hierarchical Multiagent Reinforcement Learning (CH-MARL)}, which merges hierarchical decision-making with real-time emission constraint enforcement and fairness-aware reward design. Experimental results within a digital twin environment confirmed the framework’s capability to reduce CO$_2$ outputs, respect global emission caps, and promote equitable allocations of operational burdens among vessels. 

%By integrating constraint satisfaction and fairness into multi-agent learning, CH-MARL addresses pressing sustainability requirements and fosters cooperative behaviors, advancing existing MARL approaches. Future research can focus on enhancing scalability to industrial-scale shipping networks, exploring robust partial observability solutions (e.g., attention-based or graph neural networks), and validating these methods through real-time pilot deployments. These efforts will further consolidate CH-MARL as a transformative approach for optimizing environmental, economic, and ethical objectives in multi-agent ecosystems.

In this work, we introduced the Constrained Hierarchical Multiagent Reinforcement Learning (CH-MARL) framework, a novel approach to addressing the multifaceted challenges of sustainable maritime logistics. By integrating dynamic constraint enforcement, fairness-aware reward mechanisms, and hierarchical decision-making, CH-MARL effectively balances emissions control, operational efficiency, and fairness among stakeholders. The framework was validated using a digital twin for maritime logistics, demonstrating significant improvements in energy efficiency and greenhouse gas emission reductions while promoting equitable resource allocation. These findings underscore CH-MARL’s potential to align economic objectives with environmental sustainability and fairness, making it a promising tool for fostering innovation in maritime and other industrial operations.

The hierarchical structure of CH-MARL enables scalable coordination among diverse agents, while the fairness-aware design ensures that smaller stakeholders are not disadvantaged, fostering more equitable and collaborative environments. These attributes position CH-MARL as a critical advancement in multi-agent reinforcement learning, with implications that extend beyond maritime logistics to other constrained, multi-objective domains.

\subsection{Future Research Directions}\label{sec:future}
Future research could enhance the CH-MARL framework by incorporating advanced models for partial observability, such as leveraging attention mechanisms or graph-based representations to improve agent decision-making \cite{foerster2016learning}. Additionally, integrating adversarial training or fault-tolerant approaches could bolster the framework’s robustness in the face of system disruptions or malicious agents \cite{albrecht2018autonomous}. Real-world pilot studies conducted in partnership with maritime stakeholders would provide invaluable insights into the practical implementation and scalability of the framework, bridging the gap between theoretical development and industry adoption.

While the results demonstrate the effectiveness of CH-MARL, several avenues for future research remain:
\begin{itemize} 
    \item[a.] \textbf{Enhanced Partial Observability Models:} Incorporating advanced representation techniques such as graph neural networks or attention-based mechanisms to better handle partial observability and improve decision-making in complex environments.
    \item[b.] \textbf{Adversarial and Fault-Tolerant Approaches:} Developing mechanisms to counter adversarial agents and system disruptions, ensuring the robustness and reliability of CH-MARL in high-stakes applications.
    \item[c.] \textbf{Real-World Pilot Studies:} Collaborating with industry stakeholders to implement CH-MARL in real-world settings, such as port operations and vessel routing. These studies will provide insights into the practical challenges of scaling and adapting the framework.
    \item[d.] \textbf{Generalization Across Domains:} Extending the application of CH-MARL to other constrained, multi-agent systems, such as smart grids, urban traffic management, and supply chains. Investigating domain-specific constraints and fairness metrics will further enhance the framework’s adaptability and impact.
    \item[e.] \textbf{Scalability and Computational Efficiency:} Addressing scalability challenges by optimizing computational efficiency to accommodate larger networks and real-time data processing without compromising performance or interpretability.
    \item[f.] \textbf{Integration with Emerging Technologies:} Exploring the integration of CH-MARL with technologies like blockchain for transparency in fairness enforcement and IoT for real-time data acquisition and monitoring.
\end{itemize}

By addressing these future directions, CH-MARL can be further refined and extended, contributing to the development of intelligent, sustainable, and equitable multi-agent systems across diverse industrial domains. The continued evolution of such frameworks will play a crucial role in shaping a future where technological advancements align with global sustainability goals.

%%%
\appendix
\section{Proof of Constraint Satisfaction} \label{appx:A}

\subsection{Setup and Definitions}
We analyze the Lagrangian function:
\[
\mathcal{L}(\pi, \lambda) = \mathbb{E}[R] - \lambda (\mathbb{E}[C] - \kappa),
\]
where:
\begin{itemize} 
    \item \( R \) is the reward function,
    \item \( C \) is the constraint function,
    \item \( \kappa \) is the constraint threshold, and
    \item \( \lambda \) is the dual variable associated with the constraint.
\end{itemize}

The optimization aims to maximize \( \mathcal{L} \) with respect to \( \pi \) while ensuring that \( \mathbb{E}[C] \leq \kappa \).

\subsection{Primal-Dual Updates}
\begin{enumerate} 
    \item Policy Update:
   \[
   \pi \leftarrow \pi + \eta_\pi \nabla_\pi \mathcal{L}(\pi, %\lambda),
   \]
   where \( \eta_\pi \) is the learning rate for the policy.

    \item Dual Variable Update:
   \[
   \lambda \leftarrow \lambda - \eta_\lambda \nabla_\lambda \mathcal{L}(\pi, \lambda),
   \]
   where \( \eta_\lambda \) is the learning rate for \( \lambda \).
\end{enumerate}

\subsection{Convergence:}
\begin{proof}
 \begin{enumerate} 
    \item Boundedness of \( \lambda \): 
    Using the dual update rule:
     \[
     \lambda \geq \max(0, \lambda + \eta_\lambda (\mathbb{E}[C] - \kappa)).
     \] 
   The iterative updates ensure that \( \lambda \) remains non-negative and converges to an optimal value \( \lambda^* \), as \( \mathbb{E}[C] \) is bounded by design.

    \item Stationarity of \( \mathcal{L} \): 
    At convergence, the gradients satisfy:
     \[
     \nabla_\pi \mathcal{L}(\pi^*, \lambda^*) = 0, \quad \nabla_\lambda \mathcal{L}(\pi^*, \lambda^*) = 0.
     \]

    \item Constraint Satisfaction: 
    From the stationarity of \( \mathcal{L} \), we ensure:
     \[
     \mathbb{E}[C] \leq \kappa \text{ at convergence.}
     \]
\end{enumerate}
\end{proof}

\subsection{Illustrative Example}
Consider a toy problem where agents must minimize fuel consumption (\( R \)) while adhering to an emission cap (\( C \leq \kappa \)):
\begin{itemize} 
    \item Reward \( R = -\text{fuel used} \).
    \item Constraint \( C = \text{emissions produced} \).
\end{itemize}  
Iterative updates for \( \lambda \) show that agents learn policies that maintain emissions within the cap while minimizing fuel usage.

%%---

\section{Proof of Fairness Guarantees} \label{appx:B}

\subsection{Setup and Definitions}
The fairness-aware reward is defined as:
\[
R_i' = R_i - \beta \Delta_i,
\]
where:
\begin{itemize} 
    \item \( R_i \) is the original reward for agent \( i \),
    \item \( \Delta_i = \max(0, R_\text{optimal} - R_i) \) penalizes deviations from fairness,
    \item \( \beta > 0 \) is a penalty parameter controlling the trade-off between fairness and efficiency. 
\end{itemize}

The fairness metric is:
\[
F = \min_i \frac{R_i}{R_\text{optimal}}.
\]

\subsection{Proof of Fairness Bound:}
\begin{proof}
\begin{enumerate} 
    \item Penalty Term Impact: The penalty \( \Delta_i \) reduces the gap between \( R_i \) and \( R_\text{optimal} \):
     \[
     R_i' \geq \delta R_\text{optimal},
     \]
     where \( \delta \) is the minimum acceptable fairness threshold.

    \item Trade-Off Analysis: Increasing \( \beta \) reduces variance in agent rewards:
     \[
     \text{Variance} \propto \frac{1}{\beta}.
     \]
\end{enumerate}
\end{proof}

\subsection{Illustrative Example}
In a resource allocation scenario (e.g., berth scheduling), agents optimize resource usage while maintaining fairness:
\begin{itemize} 
    \item - Without fairness, one agent may dominate the resource allocation.
    \item With fairness-aware rewards, all agents achieve comparable rewards (\( F \geq \delta \)).
\end{itemize}

%---

\section{Proof of Hierarchical Policy Convergence} %\label{appx:C}

\subsection{Setup and Definitions}
CH-MARL involves:
\begin{enumerate} 
    \item High-Level Policy (\( \pi^H \)):
    \begin{itemize} 
        \item  Operates on a coarse level (e.g., fleet management).
        \item Represents a constrained Markov Decision Process (CMDP).
    \end{itemize}  
    \item Low-Level Policy (\( \pi^L \)):
    \begin{itemize} 
        \item  Operates on finer granularity (e.g., vessel speed control).
        \item Optimizes under constraints defined by \( \pi^H \).
    \end{itemize}
\end{enumerate}
    
\subsection{Proof of Convergence:}
\begin{proof}
\begin{enumerate} 
    \item High-Level Convergence:
    \begin{itemize} 
        \item Solve the CMDP:
     \[
     \pi^H = \arg\max_{\pi^H} \mathbb{E}[R^H] - \lambda (\mathbb{E}[C^H] - \kappa).
     \]
        \item The primal-dual method ensures that \( \pi^H \) converges to an optimal policy \( \pi^{H*} \).
    \end{itemize}
   
    \item Low-Level Convergence:
    \begin{itemize} 
        \item Given constraints from \( \pi^H \), solve:
     \[
     \pi^L = \arg\max_{\pi^L} \mathbb{E}[R^L \mid \pi^H].
     \]
        \item Standard MARL techniques ensure that \( \pi^L \) converges to \( \pi^{L*} \).
    \end{itemize}

    \item Interaction Between Levels: Prove that iterative updates between \( \pi^H \) and \( \pi^L \) lead to global convergence:
     \[
     (\pi^H, \pi^L) \rightarrow (\pi^{H*}, \pi^{L*}).
     \]
\end{enumerate}
\end{proof}

\subsection{Illustrative Example}
Consider fleet-level scheduling (high-level) and vessel speed optimization (low-level):
\begin{itemize} 
    \item High-level policies allocate fleet routes while minimizing overall emissions.
    \item Low-level policies optimize vessel speeds to meet constraints from the high-level policy.
\end{itemize}
  
The iterative process converges to a globally optimal solution for fleet efficiency and emissions control.

\bibliographystyle{unsrtnat}
\bibliography{references}  %%% Uncomment this line and comment out the ``thebibliography'' section below to use the external .bib file (using bibtex) .





% \end{thebibliography}

\end{document}